\documentclass[letterpaper, 10pt, conference]{ieeeconf}  
\IEEEoverridecommandlockouts
\usepackage{bbold}
\usepackage{enumerate}
\usepackage{graphicx}
\usepackage{amsmath,amssymb}
\usepackage{color}
\usepackage{lipsum}
\usepackage{abraces}
\usepackage{tcolorbox}
\usepackage{hyperref}
\usepackage{xcolor}
\usepackage[noadjust]{cite}

\usepackage{amsmath,bm}
\usepackage{color}
\usepackage{multirow}
\usepackage{mathrsfs}
\usepackage{lineno,hyperref}
\usepackage{dsfont}
\usepackage{mathtools}
\usepackage{amssymb}
\usepackage{cite}
\usepackage{float}
\usepackage{ntheorem}
\theorembodyfont{\rmfamily}
\usepackage{amsmath,bm,bbm}
\usepackage{color}
\usepackage{subfigure}

\newtheorem{lemma}{Lemma}
\newtheorem{definition}{Definition}

\newtheorem{assumption}{Assumption}

\newtheorem{proposition}{Proposition}

\def\begmat#1{\begin{bmatrix}#1\end{bmatrix}}

\def\cale{{\cal E}}
\def\cali{{\cal I}}

\def\calh{{\cal H}}

\def\calt{{\cal T}}
\def\calb{{\cal B}}

\def\cals{{\cal S}}
\def\cale{{\cal E}}

\def\call{{\cal L}}

\def\calv{{\cal V}}

\def\pa{\mathbb{P}_a}

\def\L2e{{\cal L}_{2e}}

\def\rea{\mathbb{R}}

\def\adj{\mbox{adj}}
\def\col{\mbox{col}}
\def\hal{{1 \over 2}}
\def\et{\epsilon_t}

\def\max{{\mbox{max}}}
\def\span{{\mbox{span}}}

\usepackage[prependcaption,colorinlistoftodos]{todonotes}

\def\call{{\cal L}}

\def\calb{{\cal B}}
\def\cale{{\cal E}}

\def\caln{{\cal N}}
\def\calt{{\cal T}}

\def\hal{{1 \over 2}}

\def\col{\mbox{col}}

\def\L2{{\cal L}_2}
\def\L2e{{\cal L}_{2e}}

\def\rea{\mathbb{R}}

\def\begequarr{\begin{eqnarray}}
\def\endequarr{\end{eqnarray}}
\def\begequarrs{\begin{eqnarray*}}
\def\endequarrs{\end{eqnarray*}}
\def\begarr{\begin{array}}
\def\endarr{\end{array}}
\def\begequ{\begin{equation}}
\def\endequ{\end{equation}}

\def\begdes{\begin{description}}
\def\enddes{\end{description}}
\def\begenu{\begin{enumerate}}
\def\begite{\begin{itemize}}
\def\endite{\end{itemize}}
\def\endenu{\end{enumerate}}

\def\lef[{\left[\begin{array}}
\def\rig]{\end{array}\right]}
\def\begcen{\begin{center}}
\def\endcen{\end{center}}

\def\ki{k_{\tt I}}

\def\tr{\mbox{tr}}

\def\QED{\hfill $\blacksquare$}
\def\qed{\hfill $\triangleleft$}

\def\TAC{{\it IEEE Trans. on Automatic Control}}

\def\SCL{{\it Systems \& Control Letters}}
\def\AUT{{\it Automatica}}

\def\ijrr{{\it Int. J. of Robotics Research}}

\def\TRO{{\it IEEE Trans. on Robotics}}

\def\begmat#1{\begin{bmatrix}#1\end{bmatrix}}

\usepackage{color}

\title{\LARGE \bf
An almost globally convergent observer for visual SLAM without persistent excitation
}

\author{
Bowen Yi, Chi Jin, Lei Wang, Guodong Shi and Ian R. Manchester
\thanks{This paper is supported by the Australian Research Council. 
}
\thanks{B. Yi, L. Wang, G. Shi and I.R. Manchester are with Australian Centre for Field Robotics \& Sydney Institute for Robotics and Intelligent Systems, The University of Sydney, Sydney, NSW 2006, Australia. C. Jin is with DJI Innovation Inc., Shenzhen 518057, China. ({\tt\small bowen.yi@sydney.edu.au})}%
}

\begin{document}

\maketitle
\thispagestyle{empty}
\pagestyle{empty}

\begin{abstract}
In this paper we propose a novel observer to solve the problem of visual simultaneous localization and mapping (SLAM), only using the information from a single monocular camera and an inertial measurement unit (IMU). The system state evolves on the manifold $SE(3)\times \rea^{3n}$, on which we design dynamic extensions carefully in order to generate an invariant foliation, such that the problem is reformulated into online \emph{constant parameter} identification. Then, following the recently introduced parameter estimation-based observer (PEBO) and the dynamic regressor extension and mixing (DREM) procedure, we provide a new simple solution. A notable merit is that the proposed observer guarantees almost global asymptotic stability requiring neither persistency of excitation nor uniform complete observability, which, however, are widely adopted in most existing works with guaranteed stability.%
\end{abstract}

%

%
\section{Introduction}
\label{sec1}
%

Simultaneous localization and mapping (SLAM) is a fundamental problem widely studied in the robotics community, as well as in the field of navigation \cite{HUADIS,LOUetal}. In SLAM, two main aims are concerned to be accomplished concurrently---mapping an unknown environment, and online estimating the pose, {\em i.e.} attitude and position, of a mobile robot, thus making SLAM an important part for unmanned systems in the absence of absolute positioning systems.

The main approaches to address this problem may generally be classified into two categories. The first one is within the probabilistic and optimization framework, assuming Gaussian noises and processes, and then formulating the problem as maximum likelihood estimation, which generally has nonlinear least squares solutions, {\em e.g.} GraphSLAM \cite{THRMON} and the SLAM++ framework \cite{ILAetal}. An alternative is to obtain estimation from an asymptotic convergence viewpoint---known as filtering---by using recursive algorithms. It includes extended Kalman filter (EKF)-SLAM and many modern algorithms, {\em e.g.} FastSLAM, which combines EKF and particle filtering \cite{MONTHR}. An essential part of these algorithms is their convergence and consistency analysis, the success of which relies on first-order approximation of systems dynamics \cite{HUADIS}. It sometimes yields satisfactory performance, but may have inconsistency issues when starting from a bad initial guess, in particular for large applications, which is caused by small domains of attraction in terms of linearization---invoking high nonlinearity of the associated dynamics. In the last few years, the nonlinear control community shows great interests to SLAM, providing alternatives via nonlinear observer design to address the problem. Indeed, SLAM can be regarded as the problem of state observation of a nonlinear system living on the manifold $SE(3)\times \rea^{3n}$. On the other hand, nonlinear observer for systems on manifolds is a well established topic, with special emphasis to matrix Lie groups \cite{IZASAN,LAGetal,MAHetal}. Some very recent papers \cite{LOUetal,TANetal,GUEetal} give several solutions to the \emph{robo-centric} SLAM problem, {\em i.e.}, estimating landmark coordinates in the body-fixed frame, for which the dynamics can be transformed into linear time-varying (LTV) systems, thus avoiding the approximation error from linearization. Then, the Kalman-Bucy filter is applicable to provide \emph{globally} convergent estimation, if the robot movement guarantees uniform complete observability (UCO). It is well known that the UCO of an LTV system is equivalent to a persistency of excitation (PE) condition \cite{SASBOD}.

In this paper, we propose a new observer-based solution to visual SLAM, a case with only bearing measurement of landmarks available. Similar problems were recently studied in \cite{VANetal}, in which the authors introduce a constructive observer design to visual SLAM by lifting to a new symmetry Lie group $\mathbb{VSLAM}(3)$ in order to make the output function equivariant. Since the system is not strongly differentially observable, in order to be able to achieve asymptotic stability, some PE conditions are required for the robot trajectory. Besides \cite{LOUetal,VANetal}, some PE or UCO-type assumptions are also indispensable in some related problems, {\em e.g.}, locolization using range or direction measurements \cite{HAMSAM}, velocity estimation using normalized measurement \cite{BJOetal}, and the non-stationary Perspective-$n$-Point (PnP) problem \cite{HAMSAMtac}. Intuitively, these assumptions impose relative motion between the robot and the landmarks, the uniformity of which should hold w.r.t. time. However, such assumptions may not be satisfied in many scenarios, such as, robots stopping in specific tasks, and landmarks appearing in the field of camera only during a finite interval, which validates neither the PE nor UCO conditions. Under these circumstances, the estimates of existing SLAM observers cannot converge to their values; and the observers may even diverge in the presence of measurement noise. Overcoming these problems are one of the motivations of the paper. Our main contributions are 
\begin{itemize}
    \item[\bf C1] Showing that visual SLAM observer design can be translated into online parameter estimation, and then solved by the recently introduced parameter estimation-based observer (PEBO) \cite{ORTetalscl}, guaranteeing invariance and almost global convergence, in contrast to the locality in EKF-SLAM methods;
    \item[\bf C2] Providing a simple visual SLAM observer, which is robust {\em vis-\`a-vis} measurement noise, and enjoys low computation burden;
    \item[\bf C3] Removing the PE condition required in some recent results, {\em e.g.}, \cite{BJOetal,LOUetal,VANetal,TANetal,HAMSAMtac}, the practical importance of which can hardly be overestimated, since it is usually the case in many tasks.
\end{itemize}

The remainder of the paper is organized as follows. In Section \ref{sec2} we will present some preliminaries, notations and the kinematic model used in the paper, as well as giving the mathematical problem formulation. In Section \ref{sec3}, a novel visual SLAM observer will be designed. It will be followed by some simulation results in Section \ref{sec4}. Finally, the paper is wrapped up by a brief concluding remark. A version with full details of the proofs can be found in \cite{YIcdc}.
%
\section{\rm\textsc{Problem Formulation and Preliminaries}}
\label{sec2}
%
%

\subsection{Notations}

We use $SO(3)$ to represent the special orthogonal group, and ${\mathfrak {so}}(3)$ is the associated Lie algebra as the set of skew-symmetric matrices satisfying $SO(3)=\{R\in \rea^{3\times3}|R^\top R = I_3, ~ \det(R) =1\}$. The unit sphere is denoted as $\cals^2=\{x\in\rea^3|~|x|=1\}$. Given $a \in \rea^3$, we define the operator $(\cdot)_\times$ as 
$$
a_\times := \begmat{ 0 & - a_3 & a_2 \\  a_3 & 0 & -a_1 \\ -a_2 & a_1 & 0 } \in {\mathfrak {so}}(3) .
$$
We also consider the special Euclidean group denoted as $SE(3) = \{ \calt(R,x)\in \rea^{4\times 4}|R\in SO(3),~ x\in \rea^3\}$ with
\begequ
\label{calt}
\calt(R,x) = \begmat{R & x \\ 0 & 1}.
\endequ
The Lie algebra of $SE(3)$ is defined as 
$$
{\mathfrak{se}}(3):= \left\{ A \in \rea^{4\times 4} \Bigg|A= \begmat{\Omega_\times & v\\ 0 & 0}, \Omega_\times\in \mathfrak{so}(3), v \in \rea^3 \right\}.
$$
For any $x\in \rea^3$, $a\times x$ is the vector cross product, satisfying $a_\times x = a \times x$. We define a wedged mapping 
$
U^\vee = \begmat{\Omega_\times & v \\ 0 & 0}
$
for a vector $U:=\col(\Omega,v)\in \rea^6$. Given $A \in \rea^{n\times n}$ and $S\in \rea^{n\times n}_{\succeq 0}$, the Frobenius norm is defined as $\|A\| = \sqrt{\tr(A^\top A)}$, and $S^\hal$ and $\adj\{A\}$ represent the matrix square root and the adjugate matrix, respectively. We use $|\cdot|$ to denote Euclidean norms of vectors or its induced matrix norm. For any $x\in \rea^3/\{0\}$, its projector is defined as
$
\Pi_x :=I_3 -{1\over|x|^2}xx^\top,
$
which projects a given vector onto the subspace orthogonal of $x$. $\et$ represents exponentially decaying terms with proper dimensions. When clear from context, the arguments and subscripts are omitted. Before closing this subsection, let us recall below the notations of PE and interval excitation (IE). Obviously, IE is significantly weaker than PE, not requiring uniformity in time.
\begin{definition}\label{def1}\rm
Given a bounded signal $\phi:\rea_+ \to \rea^n$, it is
\begin{itemize}
  \item[-] $(T,\delta)$-PE, if
$
     \int_{t}^{t+T} \phi(s)\phi^\top(s) ds \succeq \delta I_n
$
    for some $T>0, \delta >0$ and all $t\ge0$.
%
%
    \item[-] $(t_0,t_c,\delta)$-IE if there exist $t_0 \ge 0$ and $t_c \ge 0$ such that
    $
    \int_{t_0}^{t_0+t_c} \phi(s)\phi^\top(s) ds \succeq \delta I_n
    $
    for some $\delta >0$. \qed

\end{itemize}
\end{definition}

\begin{table}
\begin{tcolorbox}[
colback=white!10,
coltitle=blue!20!black,  
]
\begin{center}
 {\bf Nomenclature}
\end{center}
\vspace{0.2cm}
  \renewcommand\arraystretch{1.4}
\small
\begin{tabular}{ll}
$\hat{(\cdot)}$ & Estimate of a variable or state \\
$\tilde{(\cdot)}$ & Estimation error\\
$I_{n}$ & {${n\times n}$ identity matrix} \\
$x \in \rea^3$ & Robot position \\
$v \in \rea^3$ & Translational velocity in $\{\calb\}$ \\
$\Omega \in \rea^3$ & Rotational velocity in $\{\calb\}$\\
$R \in SO(3)$ & Robot attitude matrix \\
$X \in SE(3)$ & Rigid-body pose $X:=\calt(R,x)$\\
$z_i \in \rea^3$ & Position of the $i$-th landmark in $\{\cali\}$\\
$y_i\in \rea^3$ & Bearing vector of the $i$-th landmark in $\{\calb\}$ \\
$\pa(\cdot)$ & Skew-symmetric projector $\pa(A) = {A-A^\top \over 2}$\\
$\small\{\cali,\calb,\calv,\cale\}$ & Inertial, body, virtual, and estimate frames
\end{tabular}
\end{tcolorbox}
\end{table}

\subsection{Kinematic model and visual SLAM problem}

The kinematics of a robot with rigid body is given by
\begin{equation}
\label{kinematics}
\begin{aligned}
	\dot x  = Rv, \quad 
	\dot R  = R \Omega_\times.
\end{aligned}
\end{equation}  
All the definitions of symbols and the spaces where they live in can be found in Nomenclature. We assume that there are $n$ landmarks appearing in the field view of camera, the coordinates $z_i$ of which in $\{\cali\}$ are constant, thus satisfying
\begin{equation}
\label{dyn:ldmk}
\dot z_i = 0, \quad i\in \caln :=\{1,\ldots, n\} \subset \mathbb{N}.
\end{equation}
We assume that $v$ and $\Omega$ are uniformly bounded for $t\in[0,+\infty)$, and the kinematic model is forward complete. The time derivative of $X:=\calt(R,x)$ is 
\begin{equation}
\label{dot_X}
\dot X = X U^\vee
\end{equation}
with $U=\col(\Omega,v)$ containing rotational and translational velocities. From some simple geometry identities, the point-type landmarks in $\{\calb\}$ verify
$
x^{\calb} = R^\top (z_i -x ).
$
In visual SLAM only monocular cameras and IMUs are equipped on robotics, thus only landmark bearings being measurable. Without loss of generality, we assume that the camera frame coincides with $\{\calb\}$. Then, the output is exactly the bearing
\begin{equation}
\label{output}
y_i = h_i(X,Z) = R^\top {z_i - x \over |z_i -x|} , \quad i\in \caln.
\end{equation}
It is clear that the unit vector $[z_i -x]/|z_i-x| \in \cals^2$ contains the orientation of the relative vector $(z_i-x)$. For convenience, we write 
$
Y:=[~y_1~|~\ldots~|~ y_n~]$
and
$ Z:=\col(z_1,\ldots, z_n).
$

%

\ \\
{\bf  Problem 1.} (\emph{visual SLAM observer}) Consider the kinematics \eqref{kinematics} with the output $Y$, and assume that $U$ is measured via IMUs. Design an observer
\begin{equation}
\label{obs:general}
 \begin{aligned}
 	\dot{\hat \eta} = F(\hat \eta,Y,U), \quad 
 	(\hat X,\hat Z) = H(\hat \eta,Y) \\
 \end{aligned}
\end{equation}
with $\hat X\in SE(3)$ and $\hat Z \in \rea^{3n}$, guaranteeing 
\begin{equation}
\label{convergence}
\lim_{t\to +\infty}\big[|\hat X(t) - X(t)| + |\hat Z(t) - Z| \big] = 0.
\end{equation}

%
\section{\rm\textsc{Main Results}}
\label{sec3}
%
In this section, we will design an almost globally convergent observer to achieve \eqref{convergence}, by means of PEBO.

\subsection{Key algebraic identities}
The key step in PEBOs is to generate invariant foliations among the system states and dynamic extension \cite{ORTetalscl}. Here we construct a dynamic extension
\begin{equation}
\label{dyn_ext1}
\dot X_e = X_e U^\vee, \qquad X_e:=\calt(Q,\xi)
\end{equation}
with $X_e\in SE(3)$, which can be regarded as a virtual robot in $\{\calv\}$. We have the following.
\begin{lemma}
\label{lem:1} \rm
Consider the dynamics \eqref{kinematics} and \eqref{dyn_ext1}. The dynamic extension \eqref{dyn_ext1} is forward complete, and there exists a constant matrix $X_c = \calt(Q_c,\xi_c) \in SE(3)$ satisfying  
\begin{equation}
\label{id:1}
X_e(t) \equiv X_c X(t), \quad \forall t\ge 0.
\end{equation}
\end{lemma}
\begin{proof}
Defining an error variable $E(X,X_e):=X_e X^\top$ and calculating its time derivative, we may verify the claim. See \cite{YIcdc} for more details.
%
\end{proof}

The above lemma shows that the open-loop dynamic extension \eqref{dyn_ext1} and the kinematics \eqref{kinematics} admit a linear relationship \eqref{id:1}---more precisely, there is a constant rigid transformation between the frames $\{\cali\}$ and $\{\calv\}$---by means of which we reformulate the state estimation of $X(t)$ into the problem of online \emph{constant} parameter identification of $X_c \in SE(3)$. We define the coordinates of all the landmarks $z_i$ in $\{\calv\}$ as
\begin{equation}
\label{ziv}
z_i^{v} := \xi + QR^\top (z_i -x), \quad \forall i\in \caln.
\end{equation}

A key observation is that the transformation \eqref{id:1} of the ambient from $\{\cali\}$ to $\{\calv\}$ does not change the \emph{relative} transformation from a robot to a landmark.

\begin{lemma}
\label{lem:2}\rm
Consider the dynamics \eqref{kinematics} and \eqref{dyn_ext1}. The landmark coordinates $z_i^{v}(t)$ $(i\in \caln)$ in $\{\calv\}$ are \emph{constant}, verifying
\begin{equation}
\label{const:ldmk}
 z_i^v = \xi_c + Q_cz_i.
\end{equation}
The landmark bearings in $\{\calv\}$, defined as
\begin{equation}
\label{bearing:virtual}
 y_i^v := Q^\top {z_i^v - \xi \over |z_i^v - \xi|}, \quad i \in \caln,
\end{equation}
are \emph{measurable}, {\em i.e.},
$
y_i^v(t) \equiv y_i(t), ~ \forall t\ge 0.
$
\end{lemma}
\begin{proof}
Invoking \eqref{id:1} and the definition \eqref{ziv}, we have
$$
\begin{aligned}
z_i^v & = [\xi_c + Q_cx] + [Q_c R] R^\top [z_i -x]
= \xi_c + Q_c z_i,
\end{aligned}
$$
which is clearly constant. About the second claim, we have
$
	y_i^v  = Q^\top{QR^\top(z_i - x) \over |QR^\top(z_i - x)|} 
	 = R^\top {z_i -x \over |z_i - x|} 
	 = y_i,
$
($i\in \caln$), where we have used \eqref{ziv} in the first equation.
\end{proof}

As shown in Lemma \ref{lem:2}, the landmark coordinates $z_i^v$ in $\{\calv\}$ are constant. Despite being quite simple, the above lemmata, show the estimation of time-varying systems state may be translated into online parameter estimation of $X_c$ and $z_i^v$. To the best of our knowledge, such a fact has not been used in SLAM observer design before, which, however, provides the possibility to relax significantly the PE or UCO assumptions in existing algorithms. In the sequel, we will show how to design a globally convergent observer.


\subsection{Landmarks Observer in $\{\calv\}$}

Now we are able to construct linear regressor equations (LREs) of constant $z_i^v$. From \eqref{bearing:virtual} it yields
$$
\begin{aligned}
Q y_i   = {z_i^v - \xi \over |z_i^v - \xi |}
~\implies~
Qy_i [Qy_i]^\top (z_i^v - \xi)
 = z_i^v - \xi.
\end{aligned}
$$
Noting that $y(t)$ is a unit vector, we thus obtain the LREs
\begin{equation}
\label{lre1}
q_i(t) = \Pi_{Q(t)y_i(t)} \cdot z_i^v,
\end{equation}
by defining measurable signals $q_i := \Pi_{Qy_i} \cdot \xi$, where we have used the fact $\bar y_i(t) \equiv y_i(t)$ introduced in Lemma \ref{lem:2}. It is also easy to verify the following equivalence for some $\delta'>0$.
$$
\begin{aligned}
& \quad \Pi_{Qy_i}^{\hal} \in (\delta',T)\mbox{-PE} ~
 \iff  ~
\int_{t}^{t+T} \Pi_{Qy_i} ds \succeq \delta I_3,~ \forall t\ge 0.
\end{aligned}
$$

The widely used PE or UCO-type assumptions in the existing visual SLAM observers, intuitively, require that all landmarks appear in the view filed of the camera persistently, and the robot keeps moving w.r.t. the landmarks, which in general can hardly be guaranteed in practice. In the following, we will show that with the LREs \eqref{lre1} these restrictive excitation assumptions can be relaxed using the recently introduced DREM technique \cite{ARAetaltac}. We first introduce some linear filters to generate a new LRE. To be precise, for each $i\in \caln$ introducing an $\call_\infty$ operator $\calh_i: \call^3_\infty \to \call^3_\infty$ to \eqref{lre1}, and then obtaining the extended (E)LRE 
$$
\calh_i[q_i](t) = 
\Big[
\begin{aligned}
    \calh_i[\Pi_{i,1}](t)~\Big| ~ \calh_i[\Pi_{i,2}](t)~\Big|~\calh_i[\Pi_{i,3}](t)
\end{aligned}
\Big]
z_i^v
$$
where the symbols $\Pi_{i,j}$ ($j=\{1,2,3\}$) are used with a slight abuse of notation, to denote the $j$-th column of the matrix $\Pi_{Q(t)y_i(t)}$. Here, we adopt the LTV operator of the form
\begin{equation}
\label{h:ltv}
	\calh_i(p,t)[\cdot] = {\alpha_i \over p+\alpha_i}[\Pi_{Q(t)y_i(t)}(\cdot)]
\end{equation}
with $\alpha_i >0$ and $p:=d/dt$ the differential operator. One of its state space realization is known as Kreisselmeier's ELRE \cite{KRE}---written as K-ELRE in the paper---which is given by
\begin{equation}
\label{kelre}
\begin{aligned}
	{\dot q}_i^e & = - \alpha_i q_i^e + \alpha_i \Pi_{Qy_i}^\top q_i\\
	\dot{\Phi}_i & = - \alpha_i\Phi_i+ \alpha_i \Pi_{Qy_i}^\top \Pi_{Qy_i}
\end{aligned}
\end{equation}
with the system states $(q_i^e,\Phi_i)\in \rea^3 \times \rea^{3\times 3}_{\succeq 0}$ for $i\in \caln$. It is straightforward to obtain the K-ELRE
\begin{equation}
\label{k-elre2}
 q_i^e(t) = \Phi_i(t)z_i^v + \et.
\end{equation}

Then, we mix the regressors \eqref{k-elre2} to get three decoupled, \emph{scalar} regressors for each $i \in \caln$, that is pre-multipying the adjugate matrix $\adj\{\Phi_i(t)\}$ to the both sides, thus obtaining
\begin{equation}
\label{decp-regr}
   Y_{i,j}(t) = \Delta_i(t) z_{i,j}^v + \et,  \quad i\in \caln, \; j\in\{1,2,3\}
\end{equation}
with the definitions $Y_{i,j} := \adj\{\Phi_i\} q_i^e,~ \Delta_i := \det\{\Phi_i\}$ and 
$
Y_i:=\col(Y_{i,1} , Y_{i,2} , Y_{i,3}).
$
We are in position to present a novel landmark observer in $\{\calv\}$, which only requires a strictly weaker condition than PE.

\begin{proposition} \rm 
\label{prop:2}
({\em Mapping})
The landmark observer
\begin{equation}
\label{ldmk-obs}
\begin{aligned}
	\dot{\chi}_i & =  \Delta_i ( Y_i - \Delta_i \chi_i ), \quad
	\chi_i(0) = \chi_{i,0}
	\\
	\dot{\omega}_i & = - \Delta_i^2\omega_i,\quad \omega_i(0)=1
	\\
	\dot{\hat z}_i^v & =  \gamma_i \Delta_i^e \left[ Y_i + \ki^i\big(\chi_i - \omega_i\chi_{i,0} \big) - 
	\Delta_i^e
	\hat{z}_i^v
	\right]
\end{aligned}
\end{equation}
with \eqref{kelre} and
\begequ
\label{delta_ie}
\Delta_i^e:= \Delta_i + \ki^i(1-\omega_i),
\endequ
$\gamma_i >0$ and $\ki^i>0$ ($\quad i \in \caln$), guarantees
\begin{itemize}
    \item[1)] (Internal stability) All the internal states are bounded.
    \item[2)] (Element-wise monotonicity) For $\forall t_a \ge t_b \ge 0$,
    $$
    |\hat z_{i,j}^v(t_a) - z_{i,j}^v| \le |\hat z_{i,j}^v(t_b) - z_{i,j}^v|,
    \quad j \in \{1,2,3\}.
    $$
    \item[3)] (GES under IE) Assuming that $\Pi_{Qy_i}$ is $(t_{0,i},t_{c,i},\delta_i)$-IE, the origin of the error dynamics of $\tilde z^v_i:= \hat z_i^v - z_i^v$ is globally exponentially stable.
\end{itemize}
\end{proposition}
\begin{proof}
The time derivative of $(\chi_i - z_i^v)$ is given by
$$
\dot{\aoverbrace[L1R]{ \chi_i -z_i^v }}
=
- \Delta_i^2(t) (\chi_i -z_i^v),
$$
which is an LTV dynamics, with $\Delta_i$ a scalar bounded signal. Its solution is
$$
\begin{aligned}
\chi_i(t) -z_i^v 
&
= \exp\bigg( - \int_{0}^t\Delta_i(s)^2 ds  \bigg)  (\chi_i(0) -z_i^v) 
\\
&
= \omega(t) (\chi_{i,0} -z_i^v) 
\end{aligned}
$$
by noting the solution 
$
\omega_i(t)= \exp(- \int_{0}^t\Delta_i(s)^2 ds ).
$
It yields
\begin{equation}
\label{int_regr}
  \chi_i(t) - \omega_i(t) \chi_{i,0} = [1-\omega_i(t)] z_i^v,
\end{equation}
in which we underscore that $\chi_i,\omega_i$ and $\chi_{i,0}$ ($t\ge 0$) are all available in observer design. In this way, we get new scalar LREs \eqref{int_regr} involving the ``integral'' information of $\Delta_i$. Combining \eqref{int_regr} and \eqref{decp-regr}, we obtain new linear regressors
\begin{equation}
\label{lre_mixing}
Y_i + \ki^i ( \chi_i - \omega_i\chi_{i,0} ) 
=
\Delta_i^e z_i^v,
\end{equation}
in which $z_i^v$ are the unknown constant parameters.

By defining the estimation error $\tilde z^v_i = \hat z_i^v - z_i^v$ we have\footnote{Since the exponentially decaying term $\et$ has no effect on stability, we omit it in the sequel analysis.}
\begequ
\label{dot_ziv}
\dot{\tilde z}_i^v = - \gamma_i (\Delta_i^e)^2 \tilde{z}_i^v .
\endequ
To verify the first claim, we choose the Lyapunov function 
$
V = \hal\sum_{i=1}^n \left(
|\tilde{z}_i^v|^2 + |\chi_i - z_i^v|^2 + |\omega_i|^2
\right),
$
satisfying
$$
\begin{aligned}
\dot V & = -
\sum_{i=1}^n  \left[\gamma_i (\Delta_i^e)^2 |\tilde z_i^v|^2 + \Delta_i^2 |\chi_i - z_i^v|^2 + \Delta_i^2 \omega_i^2
\right]
 \le 0.
\end{aligned}
$$
Thus, the system is internally stable.

The second claim can be easily verified invoking $\Delta_i^e$ are scalar signals. Recalling the dynamics of $\tilde z_i^v$, the last claim is equivalent to verify $\Delta_i^e \in \mbox{PE}$. To simplify presentation, we neglect the index $i$ of the IE condition from the $i$-th landmark. From the IE assumption of $\Pi_{Qy_i}$, there exist $t_0,t_c,\delta \in \rea_{+}$ such that
$
\int_{t_0}^{t_0+t_c} \Pi_{Q(s)y_i(s)}\Pi_{Q(s)y_i(s)}^\top ds
\ge \delta I_3,
$
since $\Pi_{Qy_i}$ is symmetric. Noting \eqref{kelre} in the K-ELRE, we can verify
$$
\begin{aligned}
 \Phi_i(t_0+t_c)  
 \succeq  &\int_{t_0}^{t_0+t_c} e^{-\alpha_i(t_0+t_c-s)}\Pi_{Q(s)y_i(s)}^\top\Pi_{Q(s)y_i(s)} ds
\\
 \succeq & e^{-\alpha_i t_c} \int_{t_0}^{t_0+t_c}  \Pi_{Q(s)y_i(s)}^\top\Pi_{Q(s)y_i(s)} ds
\\
 \succeq &\delta e^{-\alpha_i t_c}  I_3.
\end{aligned}
$$
It implies 
$
\Delta(t_0+t_c) \ge \delta_0 := (\delta e^{-\alpha_i t_c})^3.
$
From the continuity of differential equations, there exist small $\tau>0$ and $\varepsilon >0$ such that 
$$
\begin{aligned}\footnotesize 
 \int_{t_0+t_c-\tau}^{t_0+t_c} \Delta_i(s) ds  \ge \sqrt \varepsilon
\implies & 
\int_{0}^{t} \Delta_i^2(s) ds \ge{1 \over \tau}\varepsilon, ~\forall t\ge t_c
\\
 \implies &  1- \omega_i \ge 1-e^{ - {1\over \tau}\varepsilon}, ~\forall  t\ge t_c
\\
\implies &  \Delta_i^e= \Delta_i + \ki^i(1-\omega_i) \in \mbox{PE}
\end{aligned}
$$
where we used $\Delta_i \ge 0$, $\forall t\ge 0$ and the Cauchy-Schwarz inequality for integrals. 
\end{proof}

Some remarks are made below about the proposed design.
\begin{itemize}
    \item[\bf R1] The success of the proposed landmark observer relies on the new scalar LREs \eqref{lre_mixing}---motivated by \cite{ORTetalaut}---which satisfy the PE condition. It combines two parts, namely, the first contains information in the current small ``interval'' invoking the K-ELRE generated by the filter \eqref{kelre}, and the second one consists of historical data using an integral operation to generate the ``state transmission function'' $\omega(t)$. Here, the parameters $\ki^i$ are adopted to play a role of weighting between them. 
    $$
    \begin{aligned}\small
    \underbrace{Y_i}_{\mbox{\footnotesize``current interval''}} + ~\underbrace{ \ki^i [\chi_i - \omega_i\chi_{i,0} ]}_{\footnotesize\mbox{historical information}} 
=
\Delta_i^e(t)
z_i^v.
\end{aligned}
    $$
	\item[\bf R2] The obtained new LREs \eqref{int_regr} may be replaced by other \emph{pure integral} action, {\em e.g.},
	$
	\int_0^t Y_{i}(s) ds = \int_{0}^t \Delta_i(s)ds \cdot z_i^v,
	$
	in order to make full use of the IE condition. It, however, may cause parts of internal states in the observer unbounded as time goes to infinity, once $\Delta_i$ satisfies the PE condition.
%
    \qed
\end{itemize}

\subsection{Landmark and Pose Observer in $\{\cali\}$}

After obtaining the landmarks estimation in $\{\calv\}$, we then need to express it in $\{\cali\}$, as well as to estimate the rigid body pose $X = \calt(R,x)$. It is widely recognized that the full dynamics \eqref{kinematics}-\eqref{dyn:ldmk}, under the output functions \eqref{bearing:virtual}, is not strongly differentially observable \cite{LEEetal}---a notion widely explored in high-gain observers \cite{BES}. The underlying reason is clear that it is ambiguous to identify the origin of $\{\cali\}$ with only body-fixed bearing measurement, and angular/translational velocities, see for example \cite{ZLOFOR,VANetal} in which the estimation error converges to a \emph{quotient manifold} rather than an isolated equilibrium. There are generally two technical routes to circumvent such difficulties:
\begin{itemize}
    \item[\bf T1] Assuming the initial conditions of pose states $X(0)=\calt(R(0),x(0))$ in the inertial frame $\{\cali\}$.
    \item[\bf T2] Incorporating the measurements in $\{\calb\}$ of (at least) two vectors, the inertial coordinates of which are \emph{known} in advance. 
\end{itemize}

In the latter two known vectors provide information to transform the remainder of the task into a rigid body pose estimation problem, which is widely studied in the control literature \cite{MAHetal,BERetal,MOENAM}. However, the former coincides with the ``standard'' definition of SLAM problems. With this consideration, we will pursue the first route in the sequel. From now on, we assume the initial pose 
$
X(0) = X_\star:=\calt(R_\star, x_\star ) \in SE(3)
$
is \emph{pre-selected}, thus known, in order to ``fix'' the inertial frame $\{\cali\}$. It is clear that for \eqref{dyn_ext1} by choosing particular initial conditions we have 
$$\small
  \Big[\xi(0) = x_\star, ~ Q(0)= R_\star\Big] 
  ~\implies ~
  \Big[ X(t)  = X_e(t), ~ \forall t\ge0\Big]
$$
under \emph{ideal circumstance}. This, however, is not practically applicable, since the open-loop integral \eqref{dyn_ext1} may be problematic yielding \emph{error accumulation}. In this subsection, we introduce an approach to design a pose observer and robustify the integral operation \emph{vis-\`a-vis} measurement noise. Before continuing our design, we define the vectors 
$
r_i := z_{i+1} - z_{i}, ~ \forall i\in \caln\backslash\{n\},
$
and make the following assumptions.

\begin{assumption}
\label{ass:3ldmk}\rm
The origin of $\{\calb\}$ never coincides with any landmarks for $i\in \caln$. There are at least three landmarks $z_i$ such that 
$
r_i \times r_j \neq 0,~ \mbox{for~} i\ne j.
$
These landmarks, without loss of generality, are numbered in the first $n_\ell \ge 3$.
\end{assumption}

\begin{assumption}
\label{ass:ie}\rm 
The locomotion of the robot guarantees that $\Pi_{Q(t)y_i(t)}$ is $(t_{0,i},t_{c,i},\delta_i)$-IE ($i\in \caln$) along the trajectory $X(t)$ for all the landmarks to be mapped. Additionally, we assume a moment $T_\star \ge \max_{i\in \{1,\ldots, n_\ell\}}\{t_{c,i}\}>0$ is known.
\end{assumption}

\begin{proposition}
\label{prop:pose-observer}\rm 
({\em Localization})
Consider the kinematics \eqref{kinematics} under Assumptions \ref{ass:3ldmk}-\ref{ass:ie}. The pose observer
\begin{equation}
\label{observer:pose}\small
\left\{
\begin{aligned}\small
\dot{\bar z}_j & = \rho_j \bar \phi_j \big(\bar y_j - \bar\phi_j^\top \bar z_j \big)
\\
\dot{\hat Q}_c & =  - ( w_{\tt vis} )_\times \hat Q_c \\
\dot{\hat x} & = \hat Rv  + \sum_{j=1}^{n_\ell} \sigma_j \big(\bar z_j - \hat x - \hat Q_c^\top(\hat{z}_j^v - \xi) \big),
\end{aligned}
\right.
\end{equation}
with parameters $\rho_j, k_j, \sigma_j >0$ $(j \in \{1,\ldots, n_\ell \})$, variables $\bar r_j  = \bar z_{j+1} - \bar z_j, ~ \hat R  = \hat{Q}_c^\top Q$ and  
$
w_{\tt vis}  = \sum_{j=1}^{n_\ell -1} k_j  \hat{r}_j^v \times ( \hat Q_c \bar r_j ),
$
$$\small
\begin{aligned}
\begmat{\dot{\bar y}_j \\ \dot{\bar \phi}_j} & = \left\{
\begin{aligned}
&
\begmat{ \Pi_{Q y_j} (\xi - \xi(0) + Q(0)R_\star^\top x_\star)\\
\Pi_{Q y_j} Q(0) R_\star
},
& t \in [0, T_\star]
\\
&0, & t> T_\star
\end{aligned}
\right.
\end{aligned}
$$
%
and the landmark observer consisting of \eqref{ldmk-obs} and
\begequ
\label{ldmk-obs-alg}
\hat z_i = \hat Q_c^\top (\hat z_i^v - \xi) - \hat x, \quad i \in \caln
\endequ
achieve the task \eqref{convergence} almost globally.
\end{proposition}
\begin{proof}
The proof can be found in \cite{YIcdc}.
\end{proof}
%

To illustrate the proposed design, we give the schematic block diagram in Fig. \ref{fig:schematic}.
\begin{figure}
    \centering
    \includegraphics[width=0.42\textwidth]{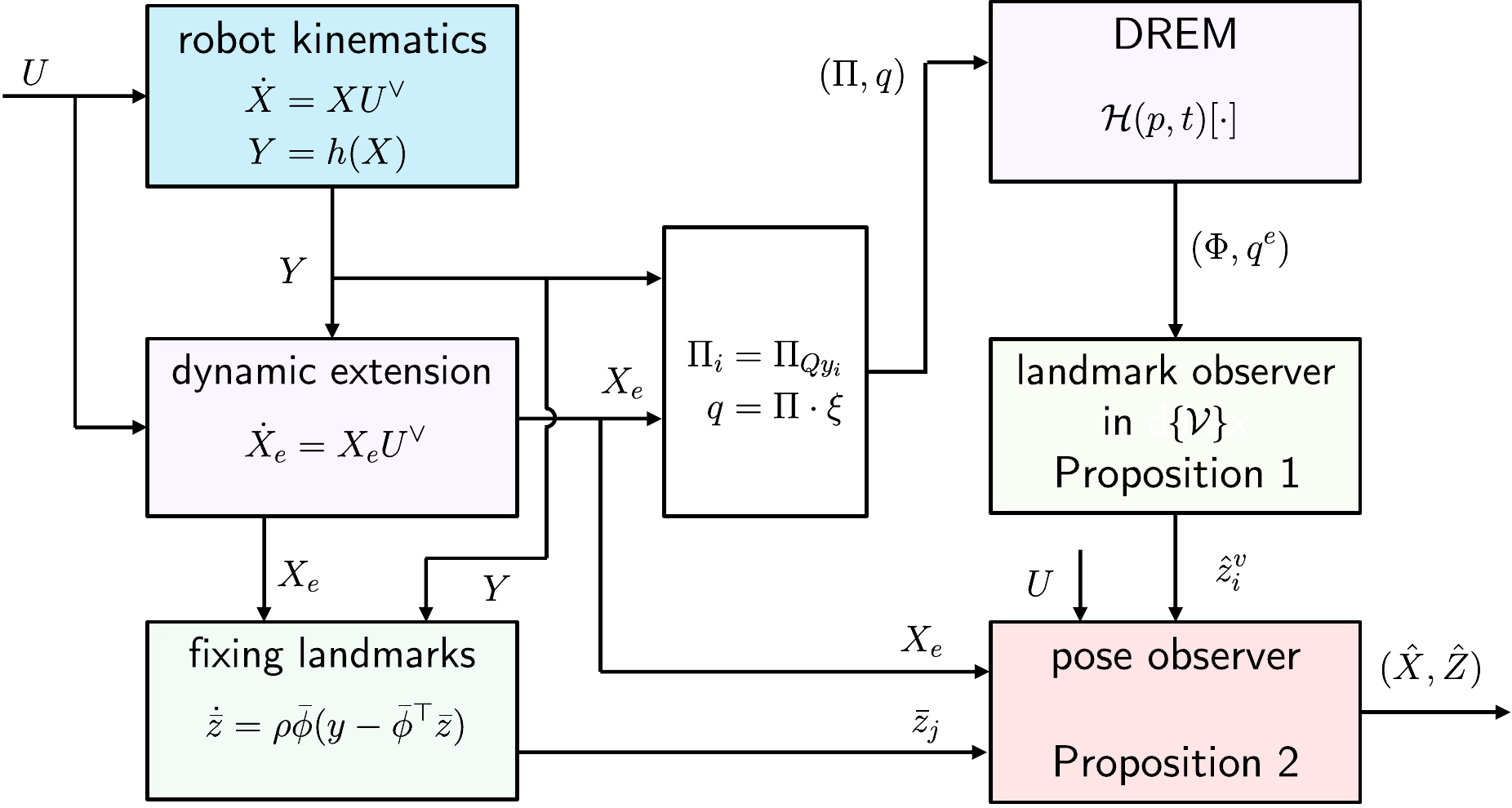}
    \caption{Block diagram of the proposed visual SLAM observer}
    \label{fig:schematic}
\end{figure}

\section{\rm\textsc{Simulations}}
\label{sec4}
 A robot is simulated from $x(0)=[1,1,2]^\top$ and the attitude $R(0) = [\cos({\pi\over 6}) , -\sin({\pi \over 6}) , 0 ; \sin({\pi \over 6}) ,\cos({\pi \over 6}) , 0 ; 0 , 0 ,1]$, and we assume that the robot stopped at $12$s with 
$$\footnotesize
v=
\left\{
\begin{aligned}
{}	&[1,0,0]^\top & t\in [0,12]\\
	&0_{3\times1}, &  t\ge 12
\end{aligned}
\right.
,
\quad
\Omega = \left\{
\begin{aligned}
{} & [0,0,-0.4]^\top, &t \in [0,12]\\
& 0_{3\times 1}, & t\ge 12 	
\end{aligned} \right.
.
$$
It guarantees that all the landmarks satisfy the IE condition with $t_0=0$ and $t_c=12$s. We consider six landmarks and $n_\ell = 3$. The initial conditions in the observer are set as
$$\small
Q(0) = \begmat{\cos({\pi\over 2}) & -\sin({\pi \over 2}) & 0\\ \sin({\pi \over 2}) & \cos({\pi \over 2}) & 0 \\ 0 & 0 & 1},
~
\hat Q_c(0) = I_3,
~
\xi(0) = \begmat{0 \\ 1 \\1 },
$$
$\hat x(0) = 0_{3\times 1}$, and $ q^e_i(0) = 0_{3\times 1}, ~ \Phi_i(0) = 0_{3\times 3}$. The observer gains are selected as $\alpha_i = 5$, $\gamma _i =100$, $k_{\tt I}^i = 20$ for $i\in \caln$, and $\rho_j = 1$ for $j=1,2,3$. Simulation was done in Matlab/Simulink, with noise added to measured velocities and bearings, generated by the block ``Uniform Random Number''. The pose estimation has a very good performance in Fig. \ref{fig:pose_no_pe} when the trajectory does not guarantee the PE condition for landmarks. The landmark observer in Proposition \ref{prop:2} can be used independently for mapping. Here, we compare it to the landmark observer in \cite{LOUetal}, which requires the system being UCO with sufficient excitation. The observer in \cite{LOUetal} studies the case that the landmark coordinates are expressed in $\{\calb\}$, {\em i.e.},
$
z_i^{\calb} = R^\top \left( z_i  - x \right) , ~ i\in \caln,
$
which are, indeed, time-varying, the simulation results of which are shown in Fig. \ref{fig:c5}. In order to make a fair comparison, we plot the evolution of the \emph{norms} of the observation errors of $\tilde z_i$ of the proposed design and $\tilde{z}^{\calb}_i$ for the one in \cite{LOUetal}, since rotation does not affect norms. The proposed landmark observer guarantees the estimation converging to a relatively small neighbourhood of their true values in the absence of the PE condition, and the small ultimate error is caused measurement noise, showing good robustness. In Fig. \ref{fig:pose_no_pe}, we observe that the estimates from the observer in \cite{LOUetal} stop converging at the moment $t_0+t_c=12$s with large errors. It is interesting to note that the estimates diverge from that moment due to the accumulation of noise, which is conspicuous by its absence in our proposed design.

\begin{figure}
   \centering
   \subfigure[$\hat{z}_i^v$ in $\{\calv\}$]{
   \includegraphics[width=0.21\textwidth]{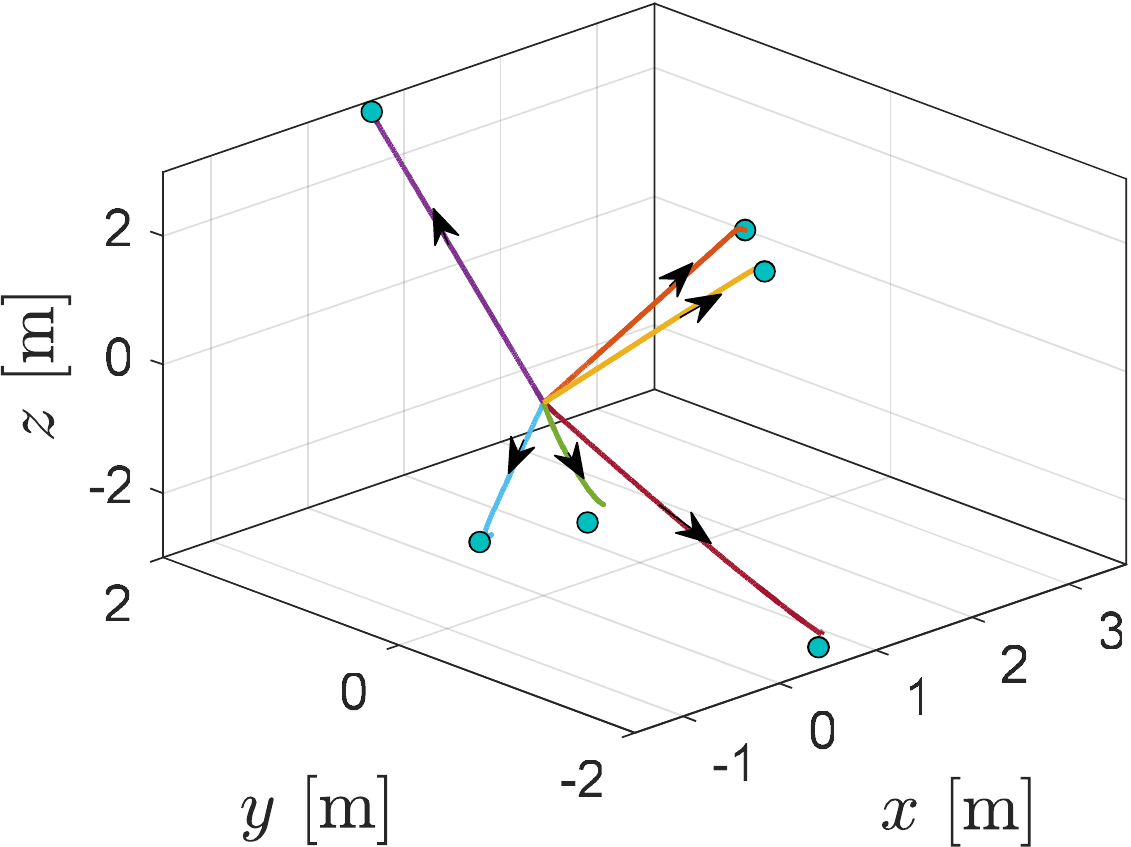}
   \label{fig:c1}
   }
   \hspace{-0.5cm}
   \subfigure[$\hat{z}_i$ in $\{\cali\}$]{
   \includegraphics[width=0.21\textwidth]{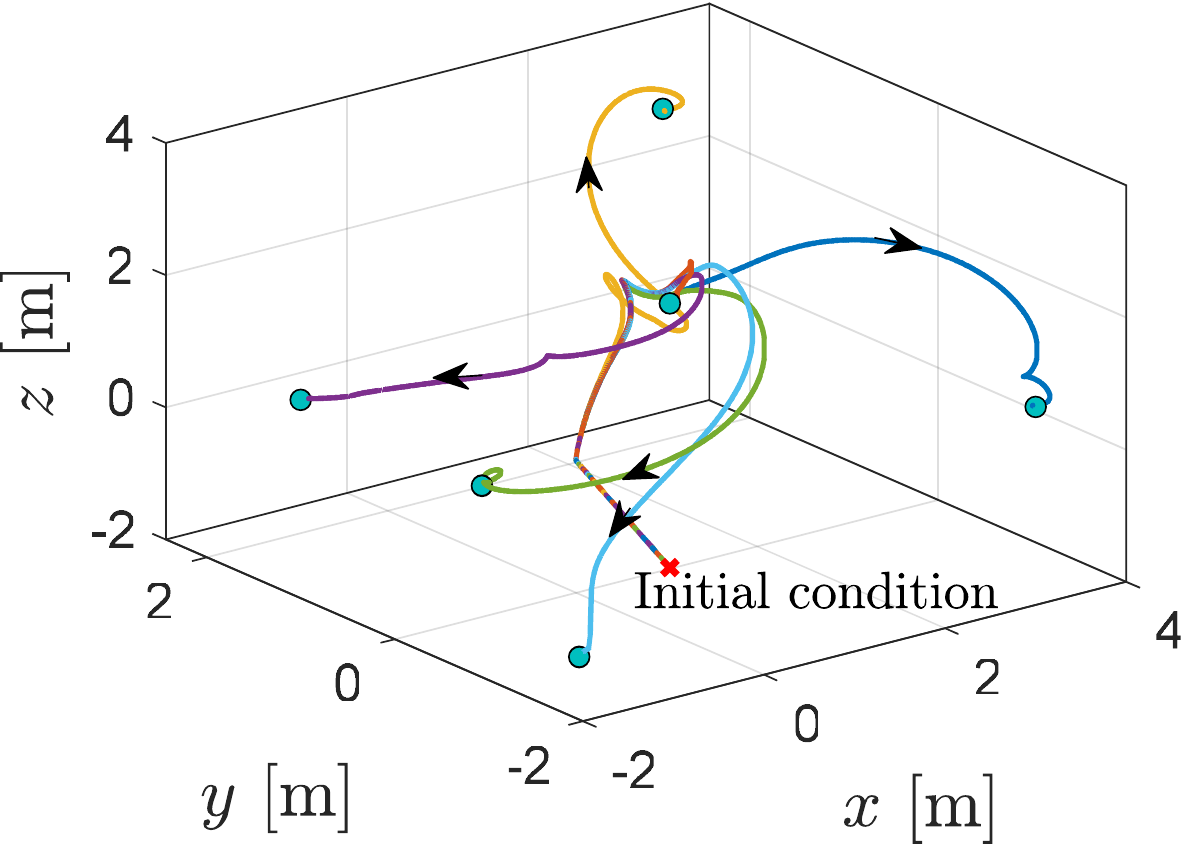}
   \label{fig:c2}
   }
    \caption{Landmark estimates $\hat{z}_i^v$ in $\{\calv\}$ and $\hat{z}_i$ in $\{\cali\}$ without PE }
    \label{fig:simulation-ldmk}
\end{figure}

\begin{figure}
   \centering
    \subfigure[the attitude $R(t)$]{
   \includegraphics[width=0.18\textwidth]{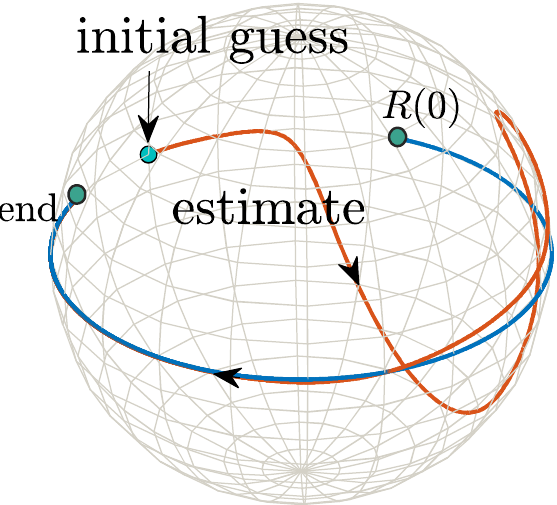}
   \label{fig:c4}
   }
   \subfigure[the position $x(t)$]{
   \includegraphics[width=0.22\textwidth]{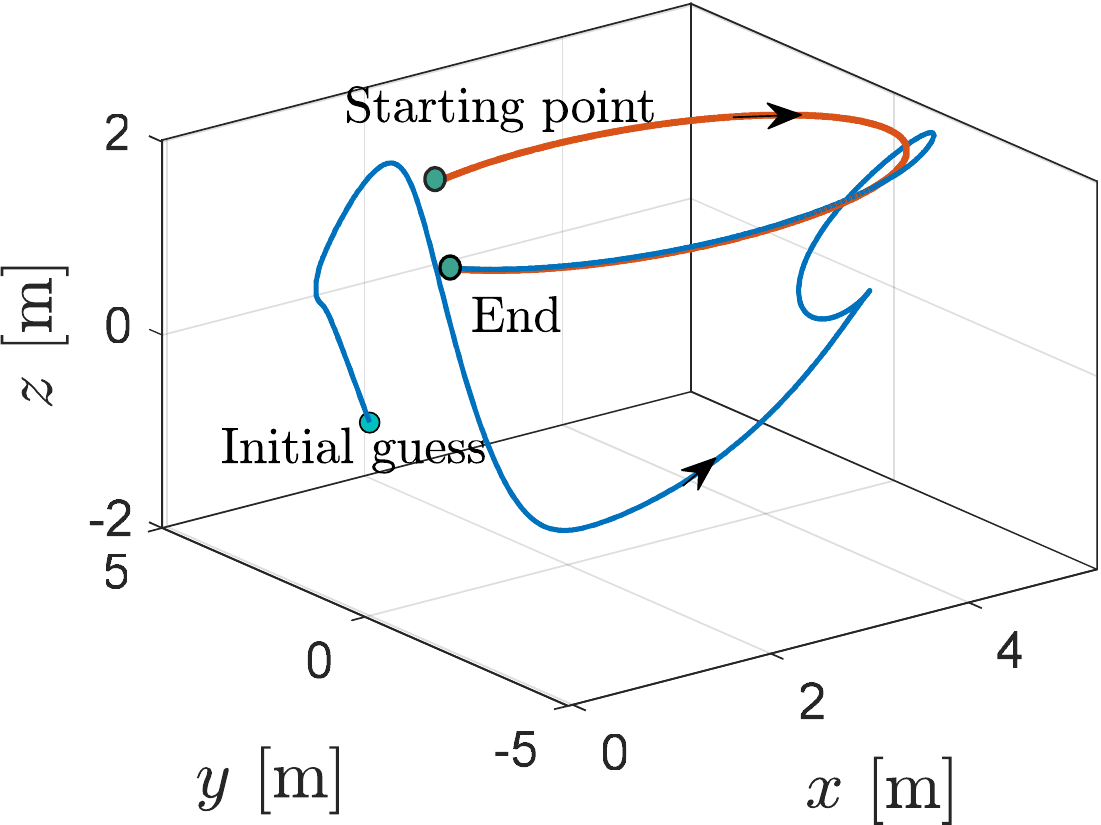}
   \label{fig:c3}
   }
    \caption{Pose $X(t)$ and its estimate $\hat X(t)$ without the PE condition}
    \label{fig:pose_no_pe}
\end{figure}


\begin{figure}[h]
   \centering
    \subfigure[Landmark paths $z_i^{\mathcal{B}}$ and their estimation paths $\hat{z}_i^{\mathcal{B}}$ in $\{\calb\}$]{
   \includegraphics[width=0.205\textwidth]{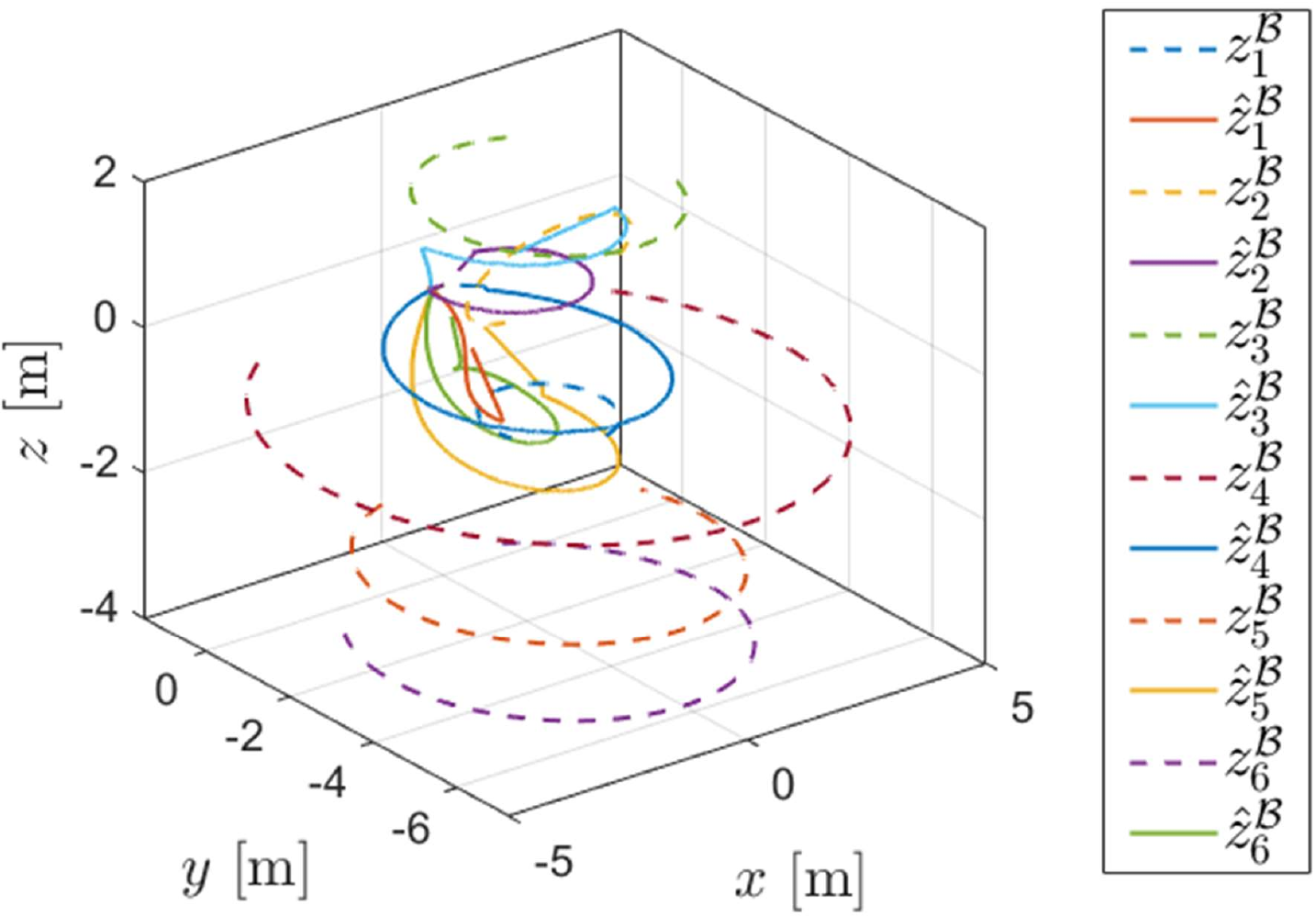}
   \label{fig:c5}
   }
   \subfigure[The landmark estimates $\hat z_1$ of the first landmark in the body $\{\calb\}$]{
   \includegraphics[width=0.205\textwidth]{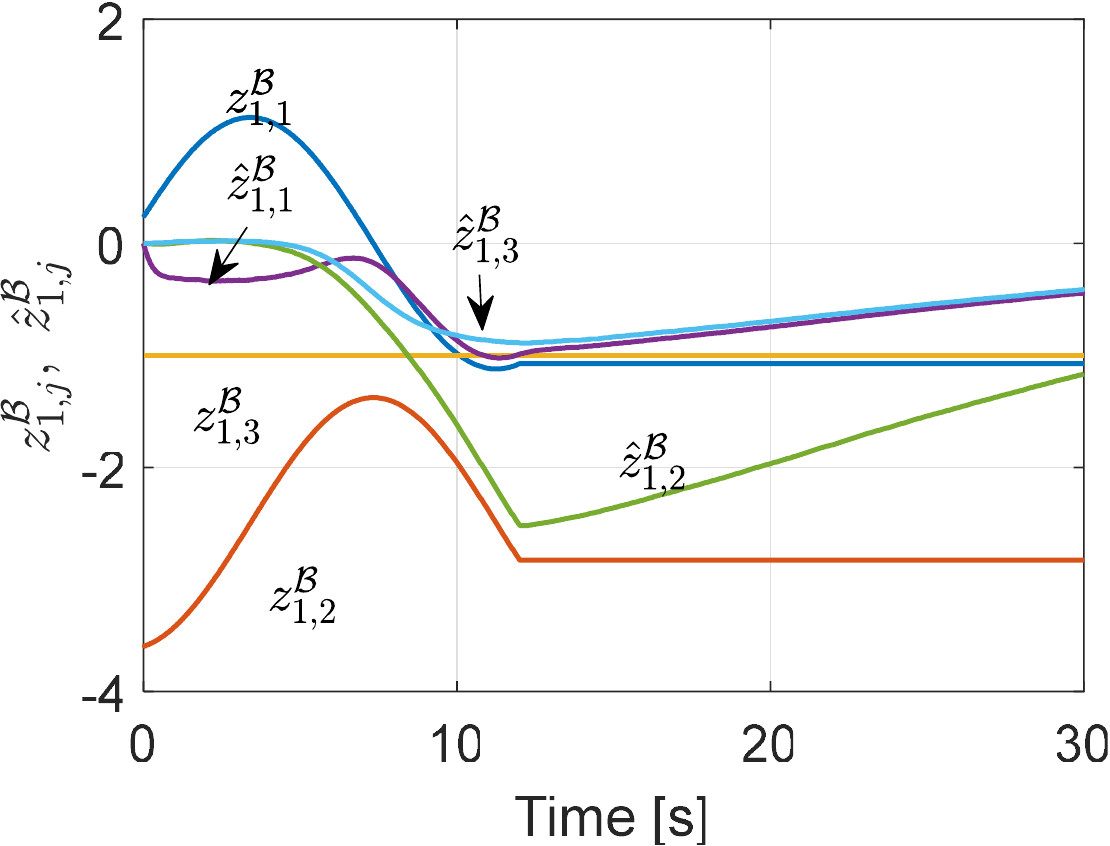}
   \label{fig:c6}
   }
   \subfigure[$|\tilde z_i^{v}|$ from the proposed landmark observer in Prop. \ref{prop:2}]{
   \includegraphics[width=0.205\textwidth]{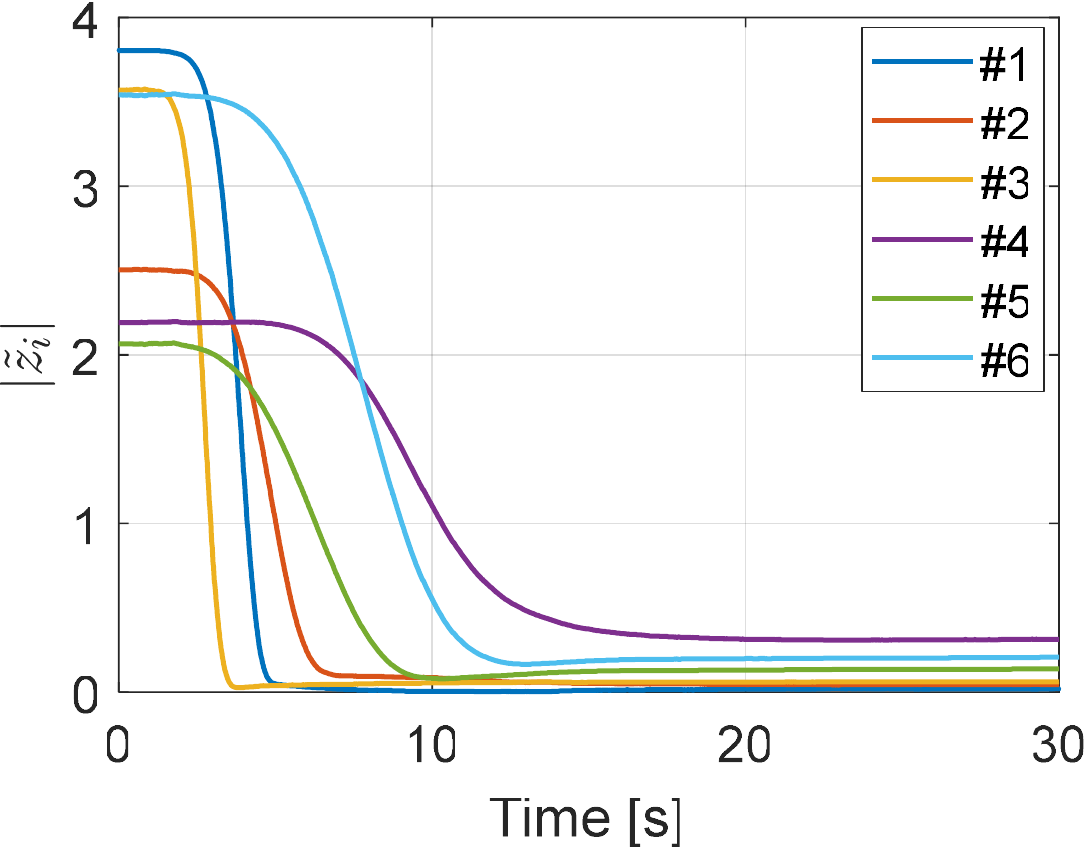}
   \label{fig:c7}
   }
    \subfigure[$|\tilde z_i^{\calb}|$ from the LTV Kalman filter in \cite{LOUetal}]{
   \includegraphics[width=0.205\textwidth]{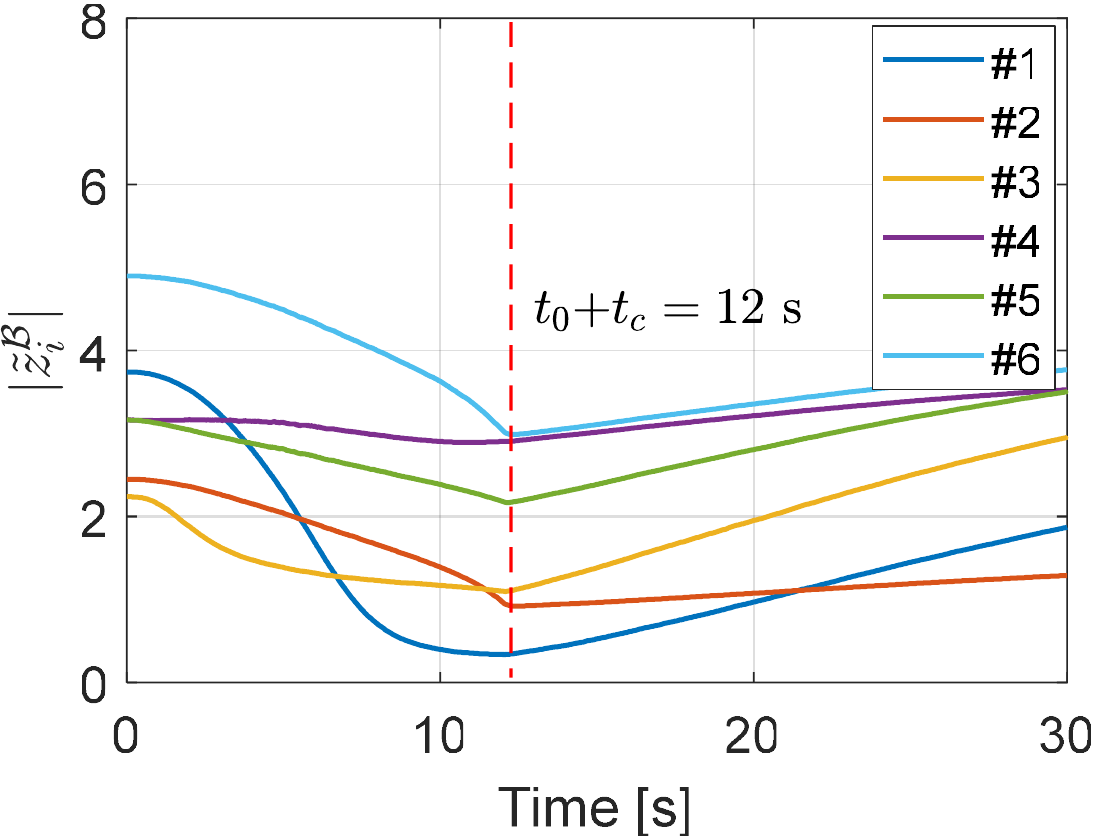}
   \label{fig:c8}
   }
   \caption{Comparison between the proposed design and  \cite{LOUetal}}
    \label{fig:comparison}
\end{figure}

\section{\rm\textsc{Conclusion}}
\label{sect5}
This paper introduces a novel visual SLAM observer design method. A key observation is that the landmarks in the frame of dynamic extension are constant, based on which we are able to get a set of linear regressors, and then transform the problem into online parameter estimation. We extend the PEBO methodology to the manifold $SE(3)\times \rea^{3n}$, the unknown ``parameters'' in our context being the landmark coordinates $z_i^v$ in $\{\calv\}$ together with the constant relative rigid transformation $X_c$. A simple constructive design is provided, with guaranteed almost global convergence, while significantly relaxing the strong PE or UCO-type conditions required in the existing literature. As future work, it is of practical interests to study the performance limitation from gyro noise and bias, as well as the approach to robustify the proposed observer design.

\section{Acknowledgement}
Partial work has been done when the second author was with I3S-CNRS, France, and he appreciates the fruitful discussions with Prof. Tarek Hamel.



\begin{thebibliography}{aa}

\bibitem{ARAetaltac}
S. Aranovskiy, A. Bobtsov, R. Ortega and A. Pyrkin, Performance enhancement of parameter estimators via dynamic regressor extension and mixing, \TAC, vol. 62, pp. 3546--3550, 2016.




\bibitem{BERetal}
S. Berkane, A. Abdessameud and A. Tayebi, Hybrid attitude and gyro-bias observer design on $SO(3)$, \TAC, vol. 62, pp. 6044--6050, 2017.

\bibitem{BES}
G. Besan\c{c}on (Ed.), {\em Nonlinear Observers and Applications}, Springer, Berlin, 2007.

\bibitem{BJOetal}
E. Bj\o{}rne, E. F. Brekke, T. H. Bryne, J. Delaune and T. A. Johansen, Globally stable velocity estimation using normalized velocity measurement, \ijrr, vol. 39, pp. 143--157, 2020.




\bibitem{GUEetal}
B. J. Guerreiro, P. Batista, C. Silvestre and P. Oliveira, Globally asymptotically stable sensor-based simultaneous localization and mapping, \TRO, vol. 29, pp. 1380--1395, 2013.

\bibitem{HAMSAM}
T. Hamel and C. Samson, Position estimation from direction or range measurements, \AUT, vol. 82, pp. 137--144, 2017.

\bibitem{HAMSAMtac}
T. Hamel and C. Samson, Riccati observers for the nonstationary PnP Problem, \TAC, vol. 63, pp. 726--741, 2018.

\bibitem{HUADIS}
S. Huang and G. Dissanayake, Convergence and consistency analysis for extended Kalman filter based SLAM, \TRO, vol. 23, pp. 1036--1049, 2007.

\bibitem{ILAetal}
V. Ila, L. Polok, M. Solony and P. Svoboda, SLAM++: A highly efficient and temporally scalable incremental SLAM framework, \ijrr, vol. 36, pp. 210--230, 2017.

\bibitem{IZASAN}
M. Izadi and A. K. Sanyal, Rigid body pose estimation based on the Lagrange-d'Alembert principle, \AUT, vol. 71, pp. 78--88, 2016.




\bibitem{KRE}
G. Kreisselmeier, Adaptive observers with exponential rate of convergence, \TAC, vol. 22, pp. 2--8, 1977.

\bibitem{LAGetal}
C. Lageman, J. Trumpf and R. Mahony, Gradient-like observers for invariant dynamics on a Lie group, \TAC, vol. 55, no. 2, pp. 367--377, 2010.

\bibitem{LEEetal}
K. W. Lee, W. S. Wijesoma and I. G. Javier, On the observability and observability analysis of SLAM, {\em Proceeding of the 2006 IEEE/RSJ Int. Conf. on Intelligent Robots and Systems}, pp. 3569--3574, 2006.

\bibitem{LOUetal}
P. Louren\c{c}o, P. Batista, P. Oliveira and C. Silvestre, A globally expoentially stable filter for bearing-only simultaneous localization and mapping with monocular vision, {\em Robotics and Autonomous Systems}, vol. 100, pp. 61--77, 2018.

\bibitem{MAHetal}
R. Mahony, T. Hamel and J.-M. Pflimlin, Nonlinear complementary filters on the special orthogonal group, \TAC, vol. 53, no. 5, pp. 1203--1218, 2008.

\bibitem{MOENAM}
A. Moeini and M. Namvar, Global attitude/position estimation using landmark and biased velocity measurements, {\em IEEE Trans. on Aerospace and Electronic Systems}, vol. 52, pp. 852--862, 2016.


\bibitem{MONTHR}
M. Montemerlo and S. Thrun, {\em FastSLAM: A Scable Method for the Simultaneous Localization and Mapping Problem in Robotics}, vol. 27, New York, NY, USA, Springer, 2007.

\bibitem{ORTetalscl}
R. Ortega, A. Bobtsov, A. Pyrkin and S. Aranovskiy, A parameter estimation approach to state observation of nonlinear systems, \SCL, vol. 85, 84--94, 2015.

\bibitem{ORTetalaut}
R. Ortega, A. Bobtsov, N. Nikolaev, J. Schiffer, D. Dochain, Generalized parameter estimation-based observers: Application to power systems and chemical-biological reactors, vol. 129, 109635, 2021



\bibitem{SASBOD}
S. Sastry and M. Bodson, {\em Adaptive Control: Stability, Convergence and Robustness}, Dover Publication, New York, 1989.

\bibitem{TANetal}
F. Tan, W. Lohmiller and J.J. Slotine, Analytical SLAM without linearization, \ijrr, vol. 36, pp. 1554--1578, 2017.

\bibitem{THRMON}
S. Thrun and M. Montemerlo, The graph SLAM algorithm with applications to large-scale mapping of urban structures, \ijrr, vol. 25, pp. 403--429, 2006.


\bibitem{VANetal}
P. van Goor, R. Mahony, T. Hamel and J. Trumpf, Constructive observer design for visual simultaneous localisation and mapping, \AUT, vol. 132, 109803, 2021.





\bibitem{YIcdc}
B. Yi, C. Jin, L. Wang, G. Shi and I.R. Manchester, An almost globally convergent obserer for visual SLAM without persistent excitation, {\em ArXiv Preprint}, 2021. (\texttt{arXiv:2104.02966})

\bibitem{ZLOFOR}
D. E. Zlotnik and J. R. Forbes, Gradient-based observer for simultaneous localization and mapping, \TAC, vol. 63, pp. 4338--4344.

\end{thebibliography}
\end{document}